\documentclass[letterpaper, 10 pt, conference]{ieeeconf}  
\IEEEoverridecommandlockouts   

\usepackage{amsfonts}
\usepackage{amsthm}
\makeatletter
\def\th@plain{%
  \thm@notefont{}
  \itshape 
}

\makeatother
\usepackage{mathtools}      
\usepackage{amssymb}

\usepackage[linesnumbered,algoruled,boxed,vlined, noend]{algorithm2e}
\usepackage{dsfont}
\usepackage{comment}
\usepackage{xcolor}
\usepackage{graphicx}
\usepackage{cite}
\usepackage[font=small,labelfont=bf]{caption}
\usepackage{subcaption}
\usepackage{algorithmic}
\usepackage{wrapfig}
\let\labelindent\relax
\usepackage{enumitem}

\usepackage{overpic}        

\usepackage[textsize=footnotesize]{todonotes}
\usepackage{xcolor}
\usepackage{xargs}
\usepackage{url}
\newcommandx{\kg}[2][1=]{\todo[linecolor=red,
			backgroundcolor=red!10,bordercolor=red,#1]{#2}}
\newcommandx{\jy}[2][1=]{\todo[linecolor=green,
			backgroundcolor=green!10,bordercolor=green,#1]{JY: #2}}
\newcommandx{\sw}[2][1=]{\todo[linecolor=blue,
			backgroundcolor=blue!10,bordercolor=blue,#1]{SW: #2}}

\usepackage[normalem]{ulem}
\def\ul#1{\uline{#1}}

\newtheoremstyle{mystyle}
  {}
  {}
  {\itshape}
  {}
  {\bfseries}
  {.}
  { }
  {}

\theoremstyle{mystyle}
\newtheorem{problem}{Problem}[section]
\newtheorem{theorem}{Theorem}[section]
\newtheorem{proposition}{Proposition}[section]
\newtheorem{lemma}{Lemma}[section]

\theoremstyle{definition}

\theoremstyle{remark}

\def\spoc{{\texttt{{SPOC}}}\xspace}
\def\osg{{\texttt{{OSG}}}\xspace}

\SetKwProg{Fn}{Function}{}{}
\SetKwComment{Comment}{$\triangleright$\ }{}

\title{\Large \bf
Sensor Placement for Globally Optimal Coverage of 3D-Embedded Surfaces
}

\author{ Si Wei Feng \and Kai Gao \and Jie Gong \and Jingjin Yu
\thanks{All authors are at Rutgers, the State University of New Jersey, Piscataway, NJ, USA.
S. W. Feng, K. Gao, and J. Yu are with the Department of Computer Science; J. Gong is with the 
Department of Civil and Environmental Engineering.  E-Mails: \{{\tt siwei.feng, kai.gao, jiegong.cee, jingjin.yu}\}\hspace*{.25em}
 @ \hspace*{.25em}rutgers.edu. 
The work is supported in part by NSF awards IIS-1734419 and IIS-1845888.}
}

\begin{document}
\maketitle
\thispagestyle{empty}
\pagestyle{empty}

\begin{abstract}
We carry out a structural and algorithmic study of a mobile sensor coverage optimization problem targeting 2D surfaces embedded in a 3D workspace. The investigated settings model multiple important applications including camera network deployment for surveillance, geological monitoring/survey of 3D terrains, and UVC-based surface disinfection for the prevention of the spread of disease agents (e.g., SARS-CoV-2). Under a unified general ``sensor coverage'' problem, three concrete formulations are examined, focusing on optimizing visibility, single-best coverage quality, and cumulative quality, respectively. After demonstrating the computational intractability of all these formulations, we describe approximation schemes and mathematical programming models for near-optimally solving them. The effectiveness of our methods is thoroughly evaluated under realistic and practical scenarios. 
\end{abstract}

\vspace{-2mm}
\section{Introduction}

We perform a systematic study of a class of mobile sensor\footnote{As will be explained, ``mobile sensor'' is used here in a broad sense.} deployment optimization problems targeting the coverage of 2D surfaces embedded in 3D domains, applicable to a broad set of practical settings, for example: (1) the selection of surveillance camera locations for maximizing the joint coverage of an art museum, (2) the deployment of mobile robots with range-based sensing apparatus for the monitoring of complex 3D terrains with quality assurances, and (3) the optimization of UVC light source locations for disinfecting interior surfaces of public indoor spaces, e.g., airplanes, buses, hospital rooms, and schools (Fig.~\ref{fig:ex}), which is of key relevance to the ongoing COVID-19 pandemic. Despite the problems' apparent differences, the above-mentioned tasks fall under the general problem of placing mobile sensors for optimizing some form of coverage. This work is devoted to providing such a unifying problem formulation, understanding its rich structural properties, and delivering effective computational methods for readily solving the multiple variations, which all have high application potentials.

\begin{figure}[!ht]
    \centering
    \includegraphics[width=0.95\columnwidth]{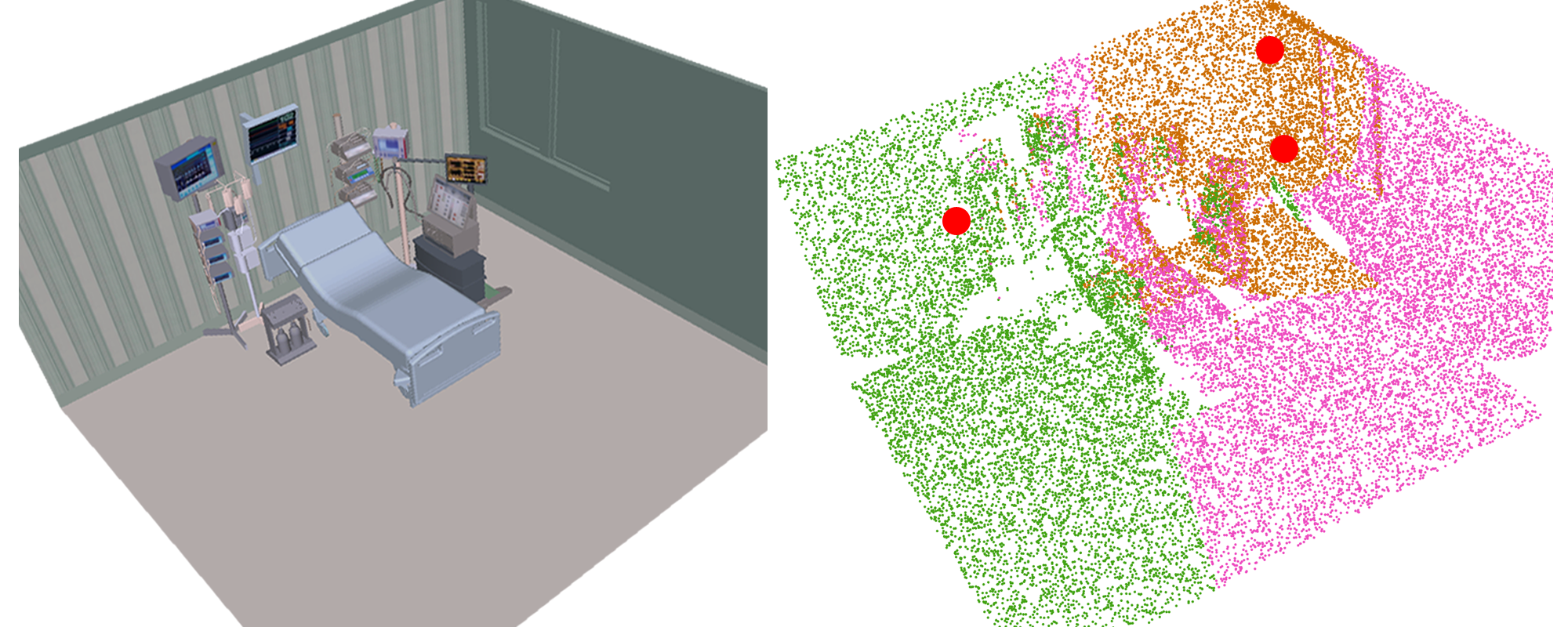}
    \vspace{1mm}
    \caption{[left] The 3D model of an intensive care unit (ICU) in a hospital. [right] A near-optimal coverage of the ICU using three UVC light sources deployed on the ceiling of the room (shown as red discs) where a minimum level of exposure dose can be guaranteed. Each color (orange, green, pink) marks the covered surface locations of a given UVC source. Areas with insufficient exposure are also readily shown as scattered white regions, which can be eliminated through adding more UVC sources.}
    \label{fig:ex}
\end{figure}

As a summary of the research and its contributions, under a general formulation, three coverage optimization problems are examined that are based on \emph{visibility}, \emph{best quality}, and \emph{cumulative quality}, respectively. The visibility model only takes into account line-of-sight sensing. In the best quality model, the coverage quality of a point in the environment is determined by the closest visible sensor, i.e., the quality is determined by distance. A max-min optimization over this quantity is performed. In the cumulative quality model, the coverage of a point is the sum of coverage by all visible ``sensors''. For this model, the area over which a minimum quality can be guaranteed is maximized. Even in simpler 2D settings, these formulations, are known to induce significant computational challenges. They are frequently NP-hard and sometimes hard to approximate, which we briefly discuss. On the algorithmic side, we show how some of these problems can be approximately solved in polynomial time and then develop general integer programming methods, assisted with local improvements, for quickly computing high quality (i.e., $(1+\varepsilon)$-optimal) solutions. Extensive evaluations are performed over multiple realistic application scenarios, confirming the effectiveness of our algorithmic solutions. 

The general formulation and specific models investigated in this work have their origins in two main lines of research: Art Gallery \cite{o1987art,shermer1992recent,o2017visibility} and studies on mobile sensor networks, e.g., \cite{howard2002mobile,cortes2004coverage,martinez2007motion,krause2008near,schwager2009decentralized,hollinger2013sampling}. Our visibility-based model has deep roots in Art Gallery problems \cite{o1987art,shermer1992recent,o2017visibility}, which commonly assume a sensor model based on line-of-sight visibility\cite{lozano1979algorithm}; a main task is to guard every point in the interior of a bounded 2D regions (a point is guarded when it is visible to at least one of the guards). Depending on the exact formulation, guards may be placed on boundaries, corners, or the interior of the region. Not surprisingly, Art Gallery problems are typically NP-hard \cite{lee1986computational}. Other than Art Gallery, 2D coverage problems with other sensing models, e.g., disc-based, have also been considered \cite{thue1910dichteste,hales2005proof,drezner1995facility,cortes2004coverage,pavone2009equitable,pierson2017adapting}. Some formulations prevent the overlapping of individual sensing ranges \cite{thue1910dichteste,hales2005proof} while others seek to ensure a full coverage which often require overlapping sensor coverage.

In an influential body of work \cite{cortes2004coverage,martinez2007motion}, a gradient-based iterative method was devised that drives multiple mobile sensors to a locally optimal coverage configuration, with 
convergence guarantees. 
Whereas \cite{cortes2004coverage,martinez2007motion} assume availability of gradients {\em a priori}, such information can also be learned \cite{schwager2009decentralized}. 
Subsequently, the method is further extended to allow the coverage of non-convex and disjoint 2D domains \cite{schwager2009optimal} and to work for mobile robots with heterogeneous capabilities \cite{pierson2017adapting}. 
In contrast to these iterative local interaction-based methods, this work emphasizes the direct computation of globally optimal solutions under challenging 3D settings. 

 Distributed sensor coverage \cite{cortes2004coverage,schwager2009decentralized} builds on the study of facility location problems \cite{weber1929theory,drezner1995facility} examining the selection of facility (e.g., warehouses) locations that minimize the cost of delivery of goods to spatially distributed customers. These are known (e.g., in operations research and computer science) as clustering problems \cite{har2011geometric}, with many variations depending on the cost structure. Our investigation benefits from the vast literature on clustering and related problems, e.g., \cite{feder1988optimal,hochbaum1985best,gonzalez1985clustering,daskin2000new,shamos1975closest}.
Clustering problems are in turn related to packing \cite{hales2005proof}, tiling \cite{thue1910dichteste}, and Art Gallery problems \cite{o1987art,shermer1992recent}.

This work is a continuation of our systematic effort \cite{FenHanGaoYuRSS19,FenYu2020RAL,FenYuRSS20} at tackling sensor coverage problems. Our earlier studies focus on 1D/2D sensors models covering 1D/2D domains, which are significantly less complicated than the 3D settings examined in the current study.

The rest of the paper is organized as follows. In Section~\ref{sec:preliminary}, we provide an umbrella problem statement and detail the three coverage models. Computational complexity of the formulations is briefly discussed. In Section~\ref{sec:algorithm}, we describe effective algorithmic solutions for solving these sensor coverage problems. Extensive evaluation results are provided in Section~\ref{sec:evaluation} containing multiple realistic application settings. We conclude the work in Section~\ref{sec:conclusion}.\label{sec:intro}

\section{Preliminaries}\label{sec:preliminary}
\subsection{Sensor Placement for Optimal Coverage: Formulations}
Let $\mathcal E \subset \mathbb R^3$ be a bounded three-dimensional workspace that is path-connected, e.g., a hilly terrain or an intensive care unit (ICU) in a hospital. We consider the problem of deploying $k$ ``mobile sensors'', $c_1, \ldots, c_k$, to guard a \emph{critical subset} $S$ embedded in the surface of $\mathcal E$, i.e., $S \subset \partial \mathcal E$ where $\partial$ is the boundary operator. For example, if $\mathcal E$ is a hospital ICU, $S$ may be a part of its interior surface. The sensors are to be deployed to achieve a globally optimal coverage of $S$ satisfying some prescribed objective, to be made more precise as the problems are further grounded for specific sensor models. We denote this broad class of problems as \emph{Sensor Placement for Optimal Coverage} (\spoc). 

The terms \emph{mobile sensor} and \emph{coverage} are used in a broad sense. Beyond traditional sensors that only collect information, we are interested in mobile robots with means to effect the environment as well. For example, a mobile robot may be equipped with a disinfecting light source (e.g., UVC) for eradicating harmful microbes (e.g., viruses and bacteria). Nevertheless, such settings can be nicely captured under a general sensor coverage formulation. 

In this study, we explore two common types of coverage models: \emph{visibility}-based and \emph{quality}-based. 
In a \emph{visibility}-based sensor coverage model, as the term suggests, a point $p \in S$ is considered covered by a sensor $c$ if $p$ is visible from $c$.
When there are more than one sensor, a point $p$ is covered if it is visible from any sensor $c_i \in \{c_1, \ldots, c_k\}$. 
In a \emph{quality}-based model, the coverage quality of a point $p \in S$ is captured by some function that potentially depends on the sensors, the point $p$, and its neighborhood in $S$. For example, one type of quality measurement can be based on the inverse of the distance between a point $p$ and its closest sensor $c$. 

Formally, we capture the different sensor models under a unified function $f(c_1, \ldots, c_k, p, \mathcal E)$ whose co-domain is non-negative reals, i.e., $f: \mathbb R^{3k + 3} \times \mathcal E \to \mathbb R_+\cup\{0\}$. In what follows, $\mathcal E$ is omitted but understood to be part of the input to $f$. In this paper, the following instantiations are considered:
\begin{itemize}[leftmargin=4mm]
    \item Let $vis(p, c) = 1$ if the point $p$ is visible to the sensor $c$. Otherwise, $vis(p, c) = 0$. For example, an omni-directional visibility model would have $vis(p, c) = 1$ if the open line segment between $p$ and $c$ does not intersect $\mathcal E$. In a model based purely on \textbf{visibility},  
    \begin{align}
    f(c_1, \ldots, c_k, p) := \max_{1\le i \le k}vis(p, c_i).\label{f:1}
    \end{align}
    \item Let the \emph{coverage quality} of a point $p$ by a sensor $c_i$ be represented as a function $\phi(p, c_i) \in \mathbb R_+\cup\{0\}$. In a \textbf{quality maximization} sensor model,  the coverage quality for a point $p$ is determined by a single best sensor:  
    \begin{align}f(c_1, \ldots, c_k, p) := \max_{1\le i \le k} \phi(p, c_i)vis(p, c_i).\label{f:2}
    \end{align}
    \item In a \textbf{cumulative quality} sensor model, the overall quality of coverage at a point $p$ is the sum of the effects of all visible sensors:
    \begin{align}f(c_1, \ldots, c_k, p) := \sum_{1\le i \le k}\phi(p, c_i)vis(p, c_i).\label{f:3}
    \end{align}
\end{itemize}

All above models have direct and practical applications. A visibility model is applicable to the deployment of a network of $360^{\circ}$ cameras for monitoring. Letting $\phi(p, c_i) = \|pc_i\|^{-1}$ ($\|pc_i\|$ denotes the distance between $p$ and $c_i$), the quality maximization model becomes the $k$-center problem \cite{weber1929theory} if we optimize $\min_p f(c_1, \ldots, c_k, p)$, a broadly applicable problem. For the cumulative quality model, when the ``sensors'' are UVC lights, one may ask the question of how to optimally place these lights to ensure the highest percentage of $S$ can be exposed to sufficient UVC light for eradicating SARS-CoV-2 and other microbes. Here, a cumulative quality model clearly makes sense. 

To fully ground the discussion that follows, we formulate three concrete optimization problems, one for each of the above-mentioned sensor models. 

\begin{problem}[Visibility Maximization]\label{p:1} Given $S \subset \partial \mathcal E$ with $\mathcal E \subset \mathbb R^{3}$, $c_i \in \mathbb R^3$, $1 \le i \le k$, and $f(c_1, \ldots, c_k, p)$ from \eqref{f:1}, determine a placement of $c_1, \ldots, c_k$ that maximizes the \ul{support of $f$}, i.e., 
$$supp(f) = \{p \in S \mid \max_{1\le i \le k}vis(p, c_i) >0\}.$$  
\end{problem}

\begin{problem}[Quality Maximization]\label{p:2} Given $S \subset \partial \mathcal E$ with $\mathcal E \subset \mathbb R^{3}$, $c_i \in \mathbb R^3$, $1 \le i \le k$, and $f(c_1, \ldots, c_k, p)$ from \eqref{f:2} with $\phi(p, c_i) = \|pc_i\|^{-1}$, determine a placement of $c_1, \ldots, c_k$ that maximizes the \ul{minimum coverage quality}, i.e., $$\min_{p \in S} \max_{1\le i \le k} \frac{vis(p, c_i)}{ \|pc_i\|}.$$
\end{problem}

\begin{problem}[Cumulative Quality]\label{p:3} Given $S \subset \partial \mathcal E$ with $\mathcal E \subset \mathbb R^{3}$, $c_i \in \mathbb R^3$, $1 \le i \le k$, and $f(c_1, \ldots, c_k, p)$ from \eqref{f:3} with $\phi(p, c_i) = \|pc_i\|^{-2}\langle \hat{n}_p, \hat{n}_{pc_i} \rangle$ where $\hat{n}_p$ is the unit normal of $S$ at $p$ and $\hat{n}_{pc_i}$ is the unit vector in the direction from $p$ to $c_i$, determine a placement of $c_1, \ldots, c_k$ that maximizes the \ul{coverage area where $f$ is above a given threshold} $\Phi > 0$,
$$supp(f - \Phi) = \{p \in S \mid \sum_{1\le i \le k} \frac{vis(p, c_i)\langle \hat{n}_p, \hat{n}_{pc_i} \rangle}{ \|pc_i\|^2} > \Phi \}.$$  
\end{problem}

\begin{wrapfigure}[11]{r}{1.4in}
  \vspace*{-1mm}
  \begin{overpic}[width=1.4in,tics=5]{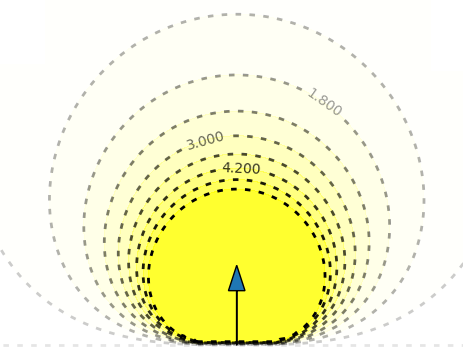}
	\end{overpic}
\vspace*{-3.5mm}
\caption{For a point with the normal shown as the arrow, light sources on the same dotted curve provide same level of exposure.}
\label{fig:exposure}
\end{wrapfigure}
An implicit assumption for Problem~\ref{p:2} to be meaningful is that an arbitrary $p \in S$ is visible to the closest sensor, which limits the choice of $S$. Problem~\ref{p:2} is a suitable model for, e.g., surveillance applications that cannot tolerate any blind spots. A generalization with less limitation can require a certain percentage of $S$, e.g., $80\%$, to have optimized coverage. Problem~\ref{p:3}, which computes coverage quality using the formula  $\langle \hat{n}_p, \hat{n}_{pc_i} \rangle \|pc_i\|^{-2}$, takes after a standard light exposure model that depends on the inverse of the squared distance and incoming light angle with respect to a local surface region (see Fig.~\ref{fig:exposure}). 

A 2D illustration of the three models is given in Fig.~\ref{fig:models}. 

\begin{figure}[ht]
    \vspace{1mm}
\centering
\includegraphics[width=0.155\textwidth]{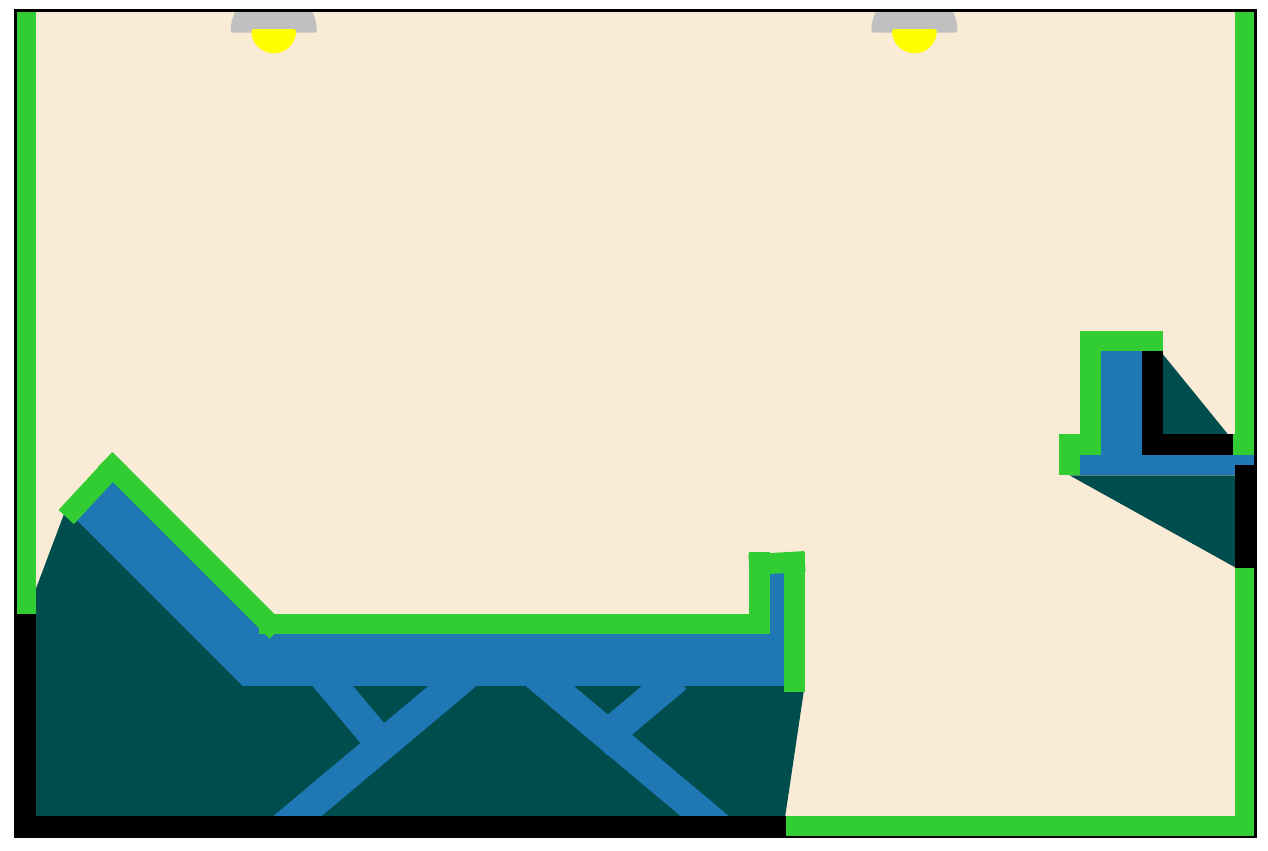}
\includegraphics[width=0.155\textwidth]{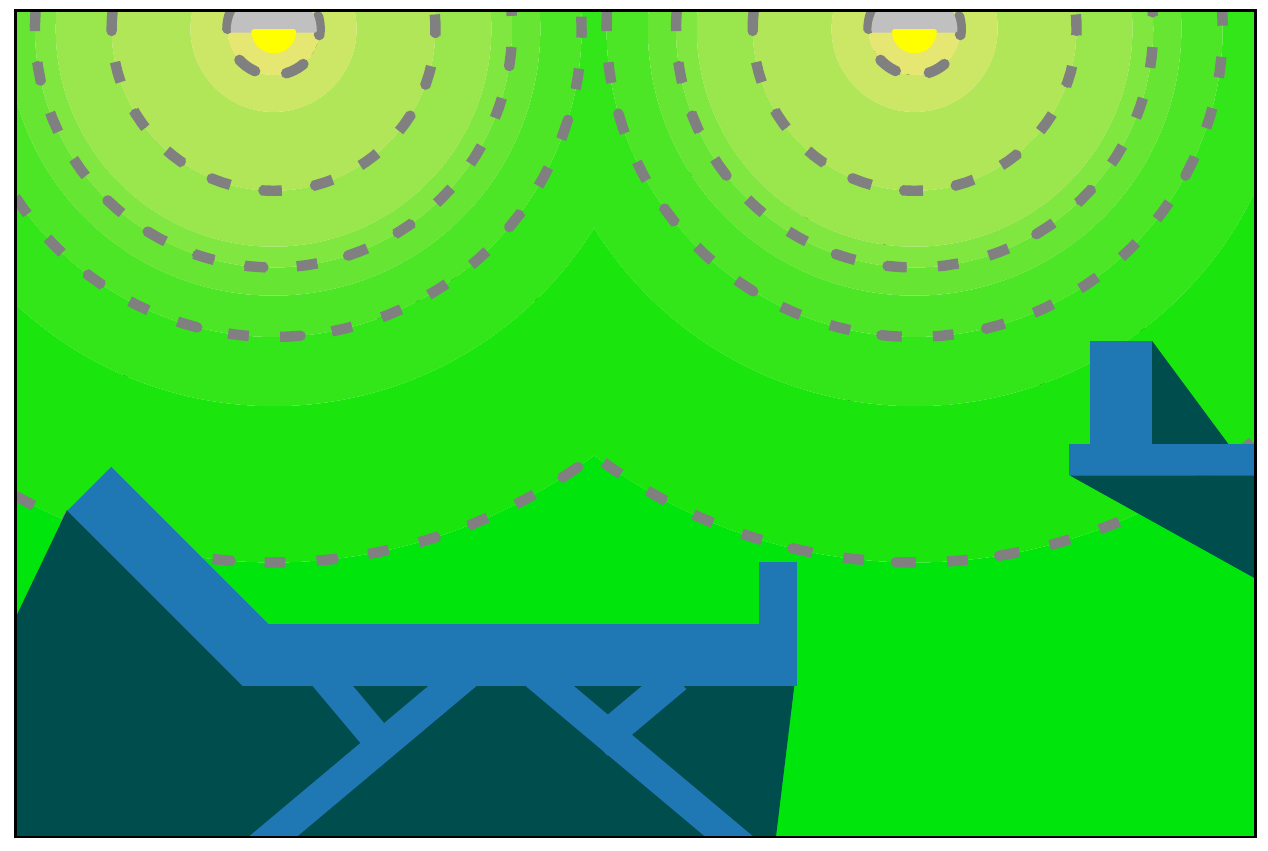}
\includegraphics[width=0.155\textwidth]{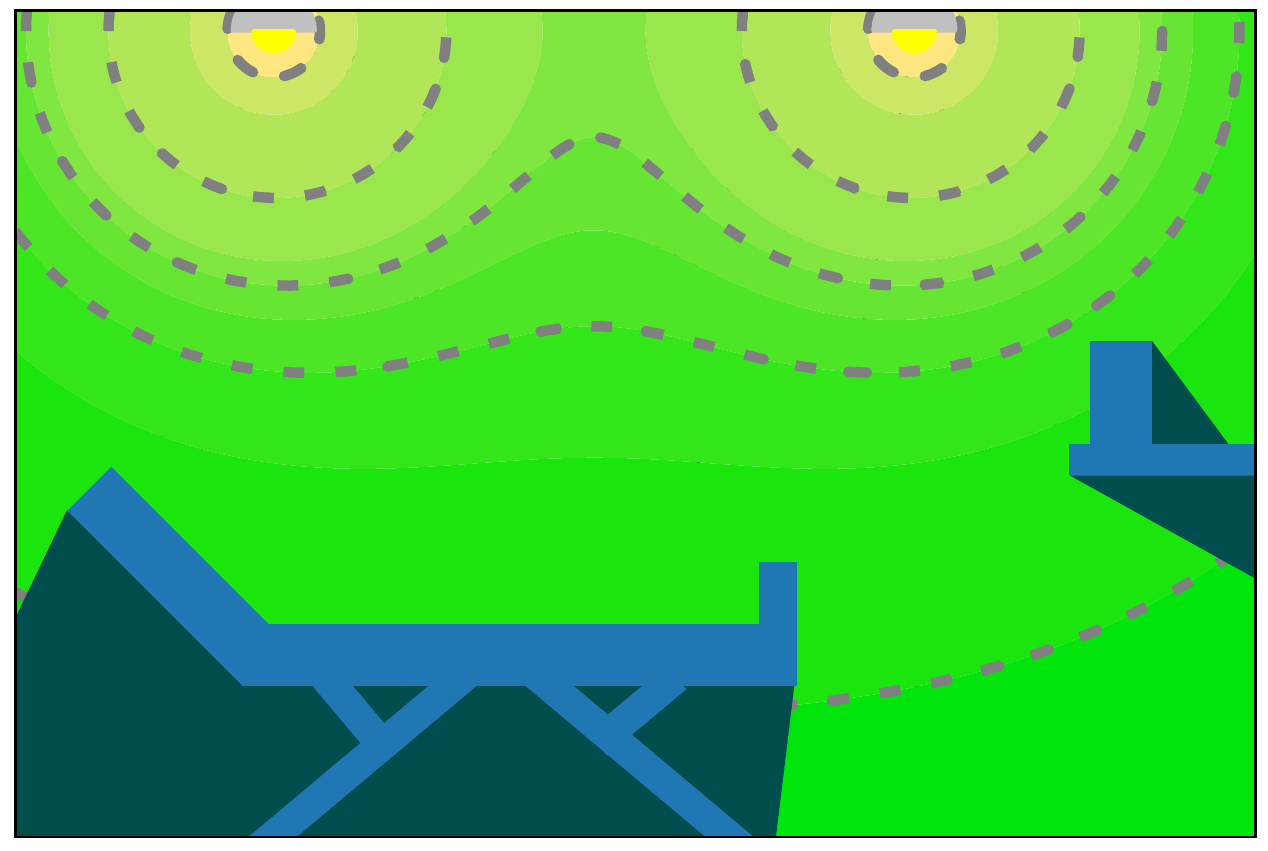}
\centering
    \vspace{1mm}
\caption{Illustration of the effects of two sensors (or light sources) under three different models. Only a 2D slice of a simple ICU with a bed and a counter is shown for clarity. [left] A visibility-based sensing model where the green segments show the visible surfaces. [middle] A quality maximization model where the dotted lines show the equal-quality level sets. $\phi(p, c_i) = \|pc_i\|^{-1}$ [right] A cumulative quality model with $\phi(p, c_i) = \|pc_i\|^{-2}\langle \hat{n}_p, \hat{n}_{pc_i} \rangle$. Again, the dotted lines show the additive quality. The impact of the surface normal is not displayed in the figure.} 
    \label{fig:models}
\end{figure}

Problems~\ref{p:1}-\ref{p:3} allow sensor locations to be anywhere in $\mathcal E$. In practice, sensor locations are often limited to a 2D surface. For example, in a museum or a bus, cameras are mounted on walls and ceilings. For drones surveying an area, there is often a preferred height to fly at. With this in mind, whereas algorithms we develop are general, the evaluation is mainly focused on practical settings where sensors locations are confined to some 2D surface.

\subsection{Computational Complexity}\label{subsec:complexity}
\vspace{-1mm}
As the computation of 3D visibility is a well-known hard problem \cite{canny1987new}
Problems~\ref{p:1}-\ref{p:3} are all computationally intractable because they all involve, as part of the solution, computation of 3D visibility sets. The involvement of multiple robots/sensors introduces additional sources of computational complexity, which we briefly discuss. 

Problem~\ref{p:1} may be viewed as an Art Gallery \cite{o1987art} problem in 3D. The basic 2D Art Gallery problem, which asks the question that how many guards with omnidirectional visibility are needed to ensure that every point in a simply-connected (2D) polygon is visible to at least one guard, is shown to be NP-hard\cite{lee1986computational}. Problem~\ref{p:1} is then also NP-hard through reducing the 2D Art Gallery problem to a 3D one by creating a third dimension that is very ``thin''.

Our recent work \cite{FenYuRSS20} shows that a 2D version of Problem~\ref{p:2}, called Optimal Set Guarding (\osg), is NP-hard to approximate within a factor of 1.152. Similarly, we may reduce \osg to Problem~\ref{p:2} by adding a thin third dimension. Therefore, through this route, we know that optimal solutions to Problem~\ref{p:2} is hard to approximate within a factor of 1.152, even when the surface $S$  is a simple polygon. 

From an instance of Problem~\ref{p:2}, reduced from an instance of \osg, we can add further 3D structures to obtain an instance of Problem~\ref{p:3} such that cumulative effect from multiple sensors are limited. That is, when sensors are forced to have no interactions, a version of Problem~\ref{p:3} that is similar to Problem~\ref{p:2} is obtained, which is again NP-hard. 

We summarize the discussion in Theorem~\ref{t:hardness}. Full proofs of these complexity results, which are too lengthy to be included here and are not as essential in comparison to the problem formulations and the algorithmic results, will be detailed in an extended version of this work. 
\vspace{-1mm}

\begin{theorem}\label{t:hardness}
Problems~\ref{p:1}-\ref{p:3} are NP-hard.
\end{theorem}
    \vspace{-2mm}


\section{Fast Computation of High-Quality Solutions}\label{sec:algorithm}
In this section, we first describe a polynomial time approximation algorithm for a restricted version of Problem~\ref{p:2}. Then, we describe a general integer linear programming framework for solving Problems~\ref{p:1}-\ref{p:3},
and local improvement techniques for enhancing solution quality.

\subsection{Polynomial Time Approximation Algorithm}
\vspace{-1mm}

In their general forms, Problems~\ref{p:1}-\ref{p:3} require the computation of 3D visibility, a hard task on its own. Due to this reason, a polynomial time algorithm with guaranteed good approximation ratio for these problems appear difficult to come by. It is an interesting question to ask whether some form of approximation scheme can be derived. Here, we show that for Problem~\ref{p:2}, if one relax the visibility requirement, i.e. letting $vis(\cdot, \cdot)\equiv 1$, then a polynomial time $(2+\varepsilon)$-approximation algorithm can be obtained. 

Taking Problem~\ref{p:2}, we examine a setup assuming that each point $p \in S$ has good visibility, i.e., $p$ is always visible to the nearest sensor. Such scenarios happen when the 3D domain does not have large curvatures that would easily block sensors' view, e.g., covering the earth with GPS satellites or using drones to survey a vineyard. To drive a specific approximation bound, 
we further assume that sensors have spherical range sensing and are in a plane of some fixed 
\begin{wrapfigure}[5]{r}{1.3in}
  \vspace*{0mm}
  \begin{overpic}[width=1.3in,tics=5]{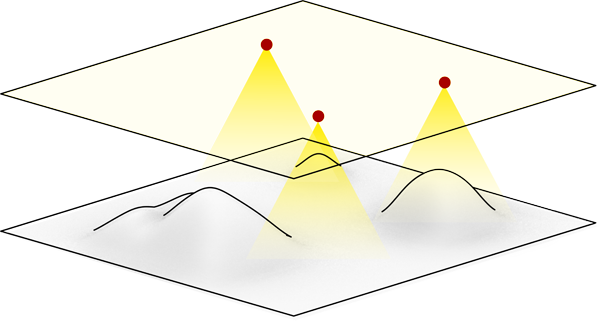}
	\end{overpic}
\vspace*{-6.5mm}
\end{wrapfigure}
height $h$ from the ground, which may be relaxed. Denote this surface as $H_C$. 
The figure on the right provides an illustration of the target environment setting. 

The main idea is to first obtain a dense sample of $S$ and then adapt $2$-approximation algorithms for the corresponding 2D setting, which requires some non-trivial reasoning. Two well-known approximation algorithms for $k$-center like problem in 2D are based on \emph{farthest point clustering} \cite{gonzalez1985clustering} and \emph{dominating set} \cite{hochbaum1985best, vazirani2013approximation}. Both of these approaches work for our purpose; we show how to work with the former.

Let a uniformly sampled set of point of $S$ be $S_N = \{o_1, \ldots, o_N\}$. We apply farthest point clustering \cite{gonzalez1985clustering} on $S_N$ as follows. As the name suggest, it picks farthest sensor locations until the number of sensors are exhausted. In the original approach, the points to be clustered are also sensor locations, which is not true here. Instead, we perform clustering in the set $S_N$ and project the selected samples to $H_C$ gradually. The relatively straightforward process is given in Algorithm~\ref{alg:greedy} (($d(\cdot$,$\ \cdot)$ denotes the distance between the inputs, one or both of which may be sets).

\begin{algorithm}
\begin{small}
    \SetKwInOut{Input}{Input}
    \SetKwInOut{Output}{Output}
    \SetKwComment{Comment}{\% }{}
    \caption{Farthest Point Clustering}
		\label{alg:greedy}
    \SetAlgoLined
		\vspace{1mm}
    \Input{$S_N$=\{$o_1, \dots, o_N$\}: $N$ sampled points on the surface $S\subset \partial \mathcal E$; $k$: number of sensors;\\
    $H_C$: a plane with a fixed height}
    \Output{$\mathcal{C}$: sensor location set}
		\vspace{1mm}
        $\mathcal{C} \leftarrow$  \{$o_1$'s vertical projection onto $H_C$\}\\
        \For{$i \gets 1$ \KwTo $k$}{
        \For{$o\in S_N$}{
        compute the distance $d(o, \mathcal{C})$ between $o$ and $\mathcal{C}$\\ 
        }
        $o \leftarrow$ point $v\in S_N$ with the largest $d(v,\mathcal{C})$\\
        $\mathcal{C} \leftarrow \mathcal{C}\ \cup \{o$'s projection onto $H_C$\}\\
        }
        \Return $\mathcal{C}$
\end{small}
\end{algorithm}

To prove the claimed $(2+\varepsilon)$-approximation bound,  
denote the optimal sensor location set and the sensor location set derived by Algorithm~\ref{alg:greedy} as $\mathcal{C}_{OPT}$ and $\mathcal{C}$, respectively. Since these are centers of spherical sensing ranges, we call them center set for short. Denote the minimum coverage radius in the spherical sensing model as $r_{OPT}$. Let $h$ be the minimum distance between surface $S$ and sensor space $H_C$, i.e. $h:=d(S,H_C)$. $r_{OPT}$ and $r_{\mathcal{C}}$ are defined as follows:
\begin{equation}
    r_{OPT}:=\max_{o\in S_N} d(o, \mathcal{C}_{OPT})
\end{equation}
\begin{equation}
    r_{\mathcal{C}}:=\max_{o\in S_N} d(o, \mathcal{\mathcal{C}})
\end{equation}

\begin{proposition}
\label{prop:algo1t}
The center set obtained by Algorithm~\ref{alg:greedy} achieves coverage radius of at most
$\sqrt{4r_{OPT}^2 - 3h^2}$.
\vspace{-1mm}
\end{proposition}
\begin{proof}


Denote the center set generated at the 
$i$th round as $\mathcal{C}_{i}$, and $r_i$ as the cluster radius $r_i := \max_{o_\tau}\min_{c_j \in \mathcal{C}_i} d(o_\tau, c_j)$. It is straightforward to observe that $r_k\leq r_{k-1} \leq \dots \leq r_1$.
Consider the relationship between the optimal center set $\mathcal{C}_{OPT}$ and the center set obtained by Algorithm~\ref{alg:greedy}, we have the following 2 cases.

Case 1: For each sphere $\mathcal{B}_{c}$ centered at a point $c\in\mathcal{C}_{OPT}$ with radius of $r_{OPT}$, the projection of $\mathcal{B}_{c}\bigcap S$ onto the sensor space $H_C$
contains exactly one point of $\mathcal{C}_k$. 


In this case, let $v$ be an arbitrary point in $S$. 
Let $c_{\alpha}$ be the nearest center to $v$ in $\mathcal{C}_{OPT}$ 
and $c_{\beta}$ be the point in $\mathcal{C}_k$ whose projection on $S$ is inside $\mathcal{B}_{c_{\alpha}}$. 
Therefore, we have: 
\begin{equation}
    d(v, c_\beta) 
    \leq 
    \sqrt{4r_{OPT}^2 - 3h^2}
\end{equation}

Case 2: There exists a sphere $\mathcal{B}_{c}$ centered at a point $c\in\mathcal{C}_{OPT}$ with radius of $r_{OPT}$, the projection of $\mathcal{B}_{c}\bigcap S$ onto the sensor space $H_C$
contains at least 2 points of $\mathcal{C}_k$. 
In this case, denote the two centers by $c_i$ and $c_j$ $(i<j)$, and their projections on $S$ are in the same sphere $\mathcal{B}_c$. As $c_j$ is added after $c_i$, then,
\begin{equation}
    \begin{split}
    r_\mathcal{C} = r_k 
    &\leq r_{j}
    \leq d(c_i, c_j\text{'s projection on } S)\\
    &\leq \sqrt{4r_{OPT}^2 - 3h^2}
    \end{split}
\end{equation}
Summarizing the two cases proves Proposition~\ref{prop:algo1t}.
\vspace{-2mm}
\end{proof}

\vspace{-0.05in}

\vspace{-1mm}
\subsection{Integer Programming-Based Algorithmic Framework}
With Problems~\ref{p:1}-\ref{p:3} being computationally intractable, a natural algorithmic alternative is mathematical programming. In \cite{FenYuRSS20}, an integer linear programming (ILP) model was shown to be effective for a 2D setting. For our 3D problems, visibility constraints must be effectively handled. We pre-compute pairwise visibility at a given sample granularity. The information is then fed to an ILP model. As the discretization granularity gets smaller, we can then realize \emph{globally optimal} $(1\pm \varepsilon)$-approximations (depending whether it is a maximization or a minimization). 

As a first step to building the ILP model, visibility information must be computed. 
We work with two discretizations, the surface $S$ to be covered and the space where the sensors may be deployed (as discussed in Section~\ref{sec:preliminary}, this is a 3D space though in practice it is frequently a 2D subset). For each pair of samples, we use a collision checker \cite{cgal:aabb-20b} to determine whether the line segments between the two samples intersects $\mathcal E$. During the process, we also compute for each sample $p\in S$ its normal $\hat{n}_p$.


%
With the visibility pre-computation performed, we are ready to construct the fully ILP models. For all three problems, recall that we have $S_N = \{o_1, \ldots, o_N\}$ for discretizing the surface $S$ through grid-based sampling. 
We use boolean variable $y_i$ to indicate whether sample $o_i$ is covered. 
Candidate sensor locations are also discretized to obtain a sample set $\{c_1, \ldots, c_M\}$, from which $k$ locations would be selected with $z_i$ indicating whether $c_i$ is selected. The ILP model for For Problem~\ref{p:1} is
\begin{gather}
    y_i   \leq \sum_{j\ s.t.\ vis(o_i, c_j) = 1} z_j   \text{\quad for each } o_i\\
    \sum_j z_j \leq k\\
    \max\ \ y_1 + \dots + y_N
\end{gather}


The cumulative quality case (Problem~\ref{p:3}) is similar. Denoting the sensing
quality between sensor location $c_j$ and surface point $p$ 
as $\phi(p, c_j) = vis(p, c_j) \cdot (
\hat{n}_p, \hat{n}_{p c_j} )/||p c_j||^2$, the ILP model may be constructed as

\begin{gather}
    y_i \cdot \Phi  \leq \sum_{j} \phi(o_i, c_j)\cdot z_j   \text{\quad for each }o_i\\
    \sum_j z_j \leq k\\
    \max\ \ y_1 + \dots + y_N
\end{gather}


For quality maximization (Problem~\ref{p:2}), the objective is to maximize 
the minimum distance of a sampled point on the surface to its nearest sensor 
location. 
For a required coverage ratio $\rho$ and radius $r$, we can verify whether it is possible to put $k$ sensors and cover $N\rho$ discretized points by checking the 
feasibility of the following model:
\begin{gather}
    y_i \leq \sum_{j\ s.t.\ ||c_j - o_i|| \leq r}  z_j \text{\quad for each } o_i\\
    \sum_j z_j \leq k\\
    N \rho \leq \sum_i y_i 
\end{gather}
A subsequent binary search can be applied to find the smallest feasible $r$. 

\subsection{Local Enhancement of Coverage Quality}
Whereas the ILP models for Problems~\ref{p:1}-\ref{p:3} support arbitrary precision, given that these problems are all computationally intractable, it can be expected that a pure ILP-based solution will only be scalable up to a certain point before an exorbitant amount computation time is needed. Inspired by the iterative update approach form \cite{cortes2004coverage}, we propose a two-phase optimization pipeline of using ILP (or the approximation algorithm for Problem~\ref{p:2}) as the first phase with a good level of \emph{global} optimality guarantee and follow that with \emph{local} improvements that can be quickly computed to enhance the initial solution. We note that, as the local improvement is enhancing a solution with a level of global optimality guarantee, the enhancement is also global in effect. For example, in Problem~\ref{p:2}, if we start with a $2$-approximation solution and obtain an initial coverage quality $r$ and subsequent local improvement reduce that to $0.75r$, then the final solution is a globally $1.5$-optimal solution.

We develop two such routines. The first is generally applicable and straightforward to implement: as the discretization level increases, we move the set of initial sensor locations (computed by Algorithm~\ref{alg:greedy} or ILP) locally, one at a time. More formally, given an initial solution $\mathcal C = \{c_1, \dots, c_k\}$, denote $S_j \subset S$ as the region covered (possibly partially when working with Problem~\ref{p:3}) by the sensor deployed at $c_j$. For each $S_j$, we try improving the quality of the solution by finding a better location for $c_j$ to cover $S_j$ at a finer resolution. Subsequently, $S_j$ can be updated based on the new $c_j$. The process may be repeated until convergence. 

The second local improvement routine is via solving a ``1-cener'' like problem and is applicable to Problem~\ref{p:2} and Problem~\ref{p:3}. Due to limited space, we omit the lengthy algorithmic details and give a high-level description. For Problem~\ref{p:2}, a sensor located at $c_j$ is ``responsible'' for visible points of $S$ that falls within a ball $B(c_j, r)$. Our improvement routine examines $S \cap B(c_j, r)$ and attempts to compute a new ball with a smaller radius that covers all of $S \cap B(c_j, r)$. The routine uses the ideas from Welzl's algorithm for computing minimum enclosing discs \cite{welzl1991smallest, Mark1997computation} and take time that is expected linear with respect to the number of samples that falls within $B(c_j, r)$, which is fairly fast. This method can be readily extended to Problem~\ref{p:3} where the spheres become ``distorted'' (Fig.~\ref{fig:exposure}).

\section{Experimental Evaluation}\label{sec:evaluation}
For each of Problems~\ref{p:1}-\ref{p:3}, extensive experimental evaluations were carried out to evaluate our proposed algorithmic solutions. Here, we present representative evaluation demonstrating the effectiveness of our methods, with a focus on three realistic settings (the ICU model from Fig.~\ref{fig:ex}, bus and subway car models shown in Fig.~\ref{fig:bus-subway}). %
For all environments, the surface $S$ is selected to be all visible surfaces not facing downward.
Due to limited space, result on the (2+$\varepsilon$)-approximation algorithm (Algorithm~\ref{alg:greedy}) is omitted (as shown in \cite{FenYuRSS20}, such methods are fast but are quite sub-optimal). The experiments were carried out on a median-end quad-core Intel i7 processor with 16GiB RAM. 
Algorithms were implemented in C++. Gurobi \cite{gurobi} was used as the Integer Programming solver. 
Source code is available at 
{\small \url{https://github.com/rutgers-arc-lab/3d_coverage}}.

\begin{figure}[!ht]
    \centering
    \includegraphics[width = .4\columnwidth]{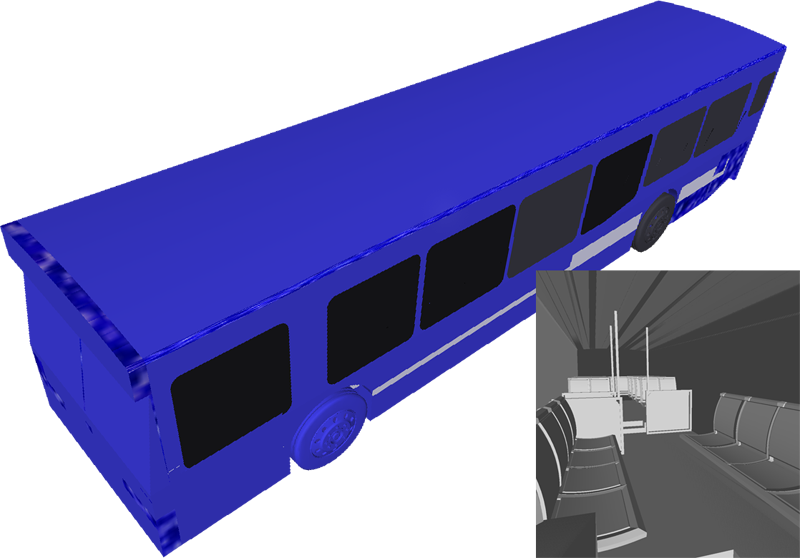}\hspace{3mm}
    \includegraphics[width = .4\columnwidth]{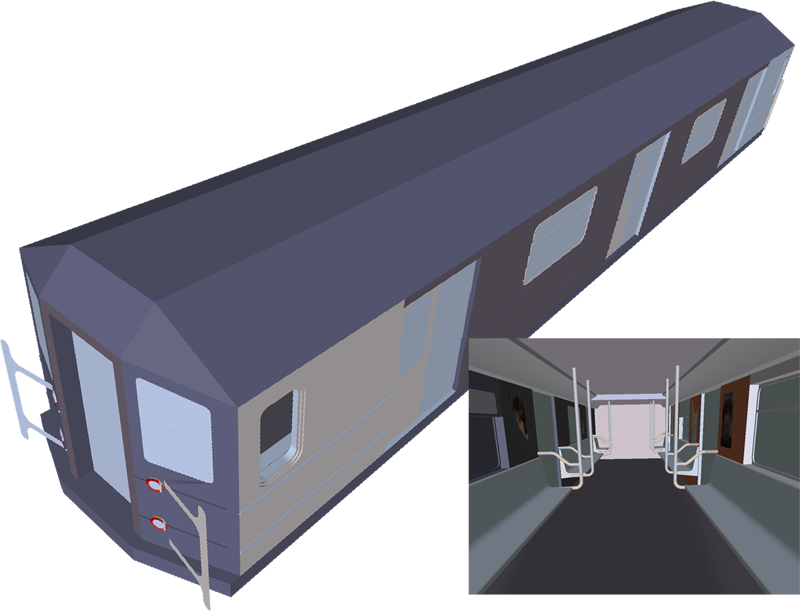}
    \caption{Realistic 3D environments used in our evaluation in addition to the ICU model from Fig.~\ref{fig:ex}. [left] A 40-foot large bus model and its interior. [right] A subway car and its interior.}
    \label{fig:bus-subway}
\end{figure}

Results on Problem~\ref{p:1}, using ILP, is given in Fig.~\ref{fig:coverage-ratio-vis}. 
For each model, $600$ candidate sensor locations and $20,000$ coverage surface points are sampled using grids. 
As expected, the surface coverage ratios increase as the number of sensors increase, approaching full coverage. We note that certain surface region is not visible, e.g., ground underneath seats in subway cars, leading to plateaus below $100\%$ coverage. The computation time is very reasonable for offline computations. The spikes in the middle of the computation time plot correspond to hard cases when the visibility coverage is about to plateau. We also observe that subway $<$ ICU $<$ bus in terms 
of computation time, which may be explained by the interior complexity of these environments. This aspect is different across the three problems.

\begin{figure}[!ht]
\vspace{1mm}
    \centering
    \includegraphics[width=0.475\columnwidth, height=1.in]{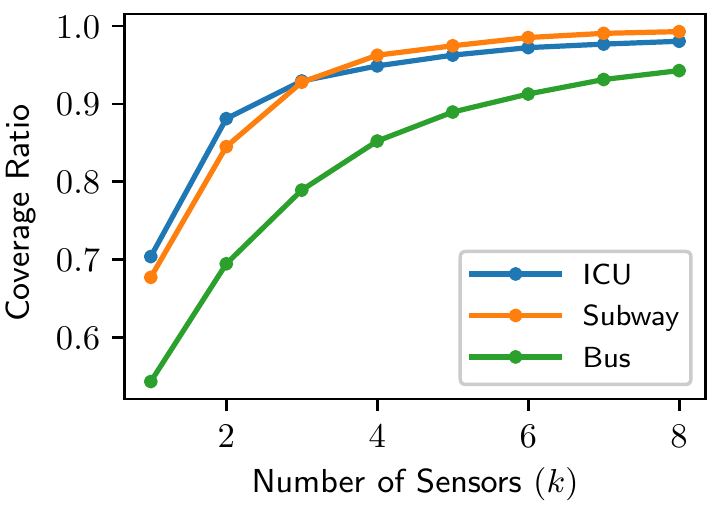}
    \includegraphics[width=0.49\columnwidth,height=1.in]{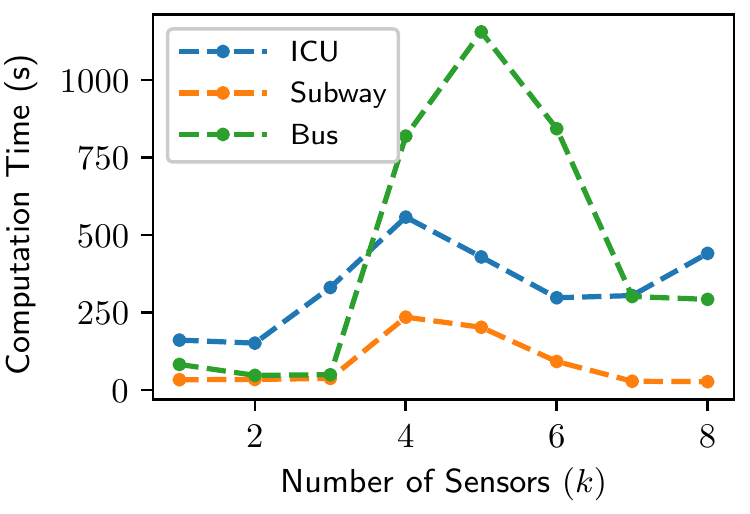}
    \caption{Coverage quality and computation time for Problem~\ref{p:1} for the three environments as the number of sensors change.}
    \label{fig:coverage-ratio-vis}

\end{figure}

\begin{wrapfigure}[7]{r}{1.1in}
  \vspace*{0mm}
  \begin{overpic}[width=1.1in,tics=5]{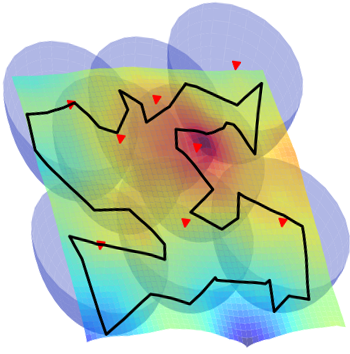}
	\end{overpic}
\vspace*{-6.5mm}
\end{wrapfigure}
In evaluating Problem~\ref{p:2}, we first examine a case where mobile sensors
(e.g., camera drones) are deployed to cover a synthetic terrain with relatively 
small curvature, i.e., $vis(\cdot, \cdot) \equiv 1$. An illustration of the 
setting is given in the figure on the right, where the color indicates the height 
of the terrain. The sensors (8 red triangles in the figure) are placed
at a fixed height above the terrain and must guard the region enclosed 
in the black curve. Spherical range sensing is assumed. For the setup 
($600$ sensor locations and $20,000$ surface points), computation time 
and solution quality as the number of sensors changes are listed in the table. Computation time decreases as the number of sensors 
increases, indicating the problem is harder when sensors are too few to provide 
a good coverage. It also shows that the ILP running time does not depend positively
on sensor quantity. The quality (smaller is better) increase becomes minimal
as sensor quantity reaches $10$.

\vspace{1mm}
\begin{table}[!ht]
    \centering
    \begin{tabular}{|c|c|c|c|c|c|c|c|c|}
    \hline
        \#Sensors   & 2     &  4    & 6      & 8     & 10    & 12    & 14    & 16 \\
        \hline
        Time (s)    & \;42.9\;\;     &  \;25.4\;\;      & \;18.5\;\;     & \;12.7\;\;    & \;13.1\;\;    & \;11.3\;\;  & \;7.64\; &   \;6.70\; \\ 
        \hline
        Radius  & 5.10     &  3.31    & 2.76      & 2.43     & 2.27    & 2.16    & 2.07    & 1.99 \\
        \hline
    \end{tabular}
    \label{tab:Terrain}
\end{table}

In a second evaluation of Problem~\ref{p:2}, visibility is considered with the 
optimization coverage ratio set to $80\%$. That is, at least $80\%$ of the
maximum visible target surface (for a given $k$) will be guaranteed the achieved coverage 
quality. The result, summarized in Fig.~\ref{fig:coverage-ratio-mq}, again 
demonstrates a negative correlation between the computational time and the number 
of sensors. Here, however, the computational time is $20+$ times  more 
than when having full visibility.

\begin{figure}[!ht]
    \centering
    \includegraphics[width=.46\columnwidth, height=0.98in]{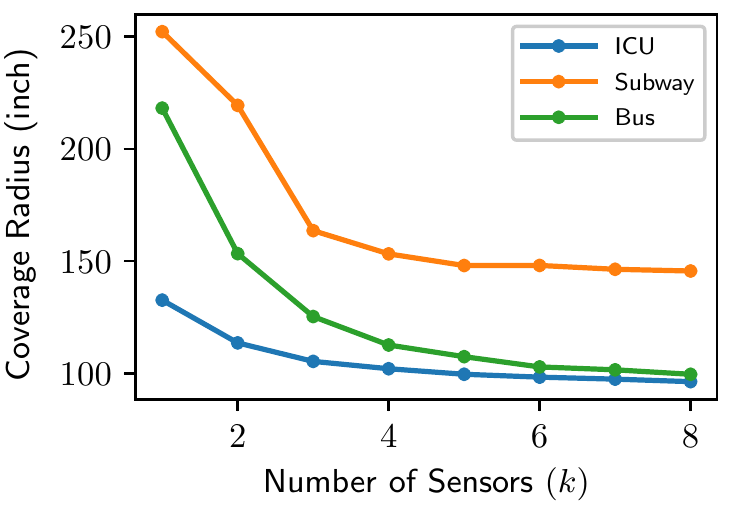}
    \includegraphics[width=.49\columnwidth, height=1.in]{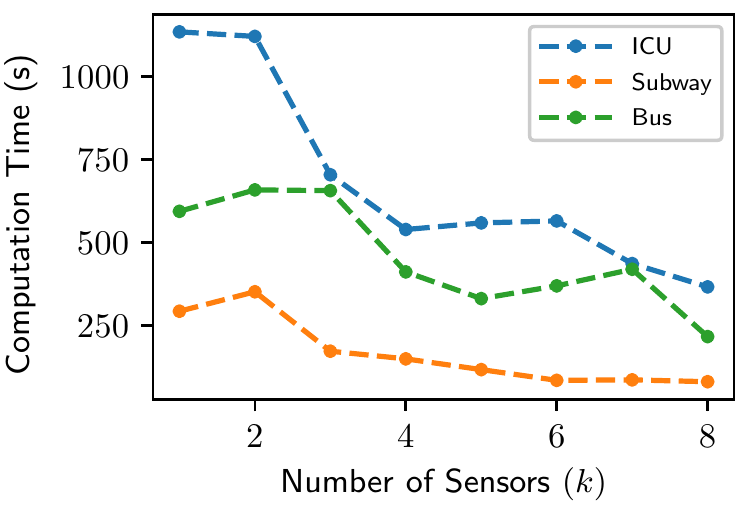}
    \caption{Coverage quality (lower is better) and computation time for Problem~\ref{p:2}
    for the three environments, as sensors increase.}
    \label{fig:coverage-ratio-mq}
\end{figure}

Our last benchmark on Problem~\ref{p:2} tests the effectiveness of the local improvement following the resolution of a coarsely generated ILP, at $60$ candidate sensor locations and $1000$ surface samples (Fig.~\ref{fig:coverage-ratio-cu}). The right figure shows much faster computation time. The left figure shows that the  faster method does a decent job for the bus environment (other environments have similar outcomes). The result suggests which method to use would depend on whether computational time or solution optimality is more important to the task at hand. We note that the local improvement method does not help improve the ILP result at the higher resolution. 

\begin{figure}[!ht]
    \centering
    \includegraphics[width=.48\columnwidth, height=1.in]{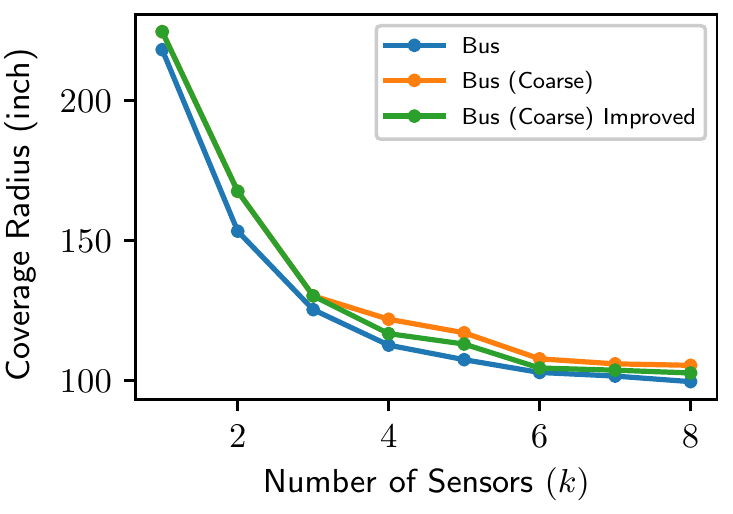}
    \includegraphics[width=.48\columnwidth, height=1.in]{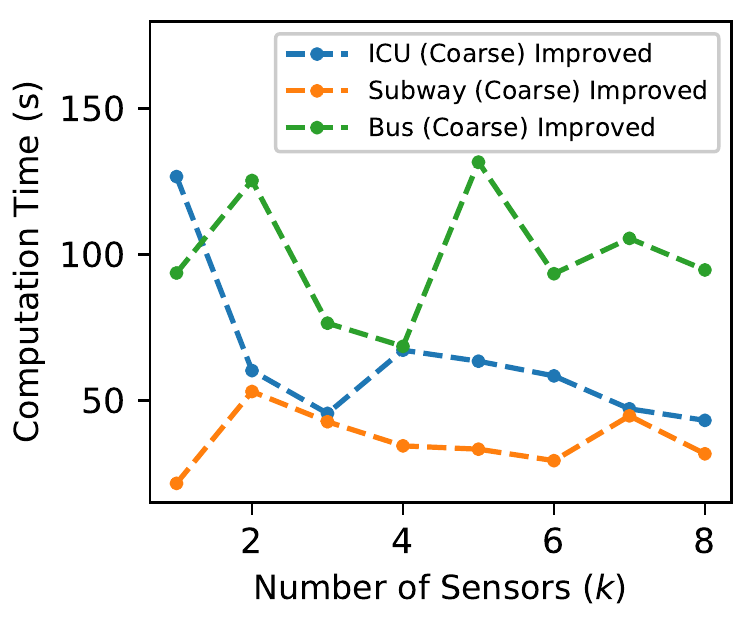}    
    \caption{ [left] Solution quality (lower is better) for the bus environment. 
    The first curve (Bus) is the same as that from Fig.~\ref{fig:coverage-ratio-mq}. 
    [right] Computation time using coarse ILP + local improvement.}
\label{fig:coverage-ratio-cu}
\end{figure}

For Problem~\ref{p:3}, computation becomes more demanding. At the specified 
discretization level, most ILP models did not complete the optimization process 
in $10$ minutes. The intermediate quality result is given in Fig.~\ref{fig:computation-time-3}, on the left (the same threshold, selected to make the computation challenging, is used for all three environments), where the lines 
corresponds to the coverage ratio returned by the ILP model and the attached vertical bars show the reported optimality gap. The crosses show the updated ratio after local improvement is carried out (the triangles will be explained shortly). 
The subway data was shifted to the left to improve readability.
For the bus, we see that the local improvement does help improve solution optimality, suggesting it is the most difficult problem. For the other two, it appears that the solution by the ILP model is already quite optimal, but the ILP solver still needs time to close the gap from the above. The subway case has worse coverage by the same number of sensors because it is larger. If we run a coarse ILP model ($60$ sensor candidates, $1000$ surface sample) for one minute and then do local improvement, we get coverage ratios shown as the triangles in Fig.~\ref{fig:computation-time-3}, left. Fig.~\ref{fig:computation-time-3}, right shows the total computation time used. We observe that except for the challenging bus model, the faster method achieves essentially identical optimality as running ILP at higher resolutions. Subway costs most time here because it is the largest. 

\begin{figure}[!ht]
    \centering
    \includegraphics[width=.48\columnwidth, height=1in]{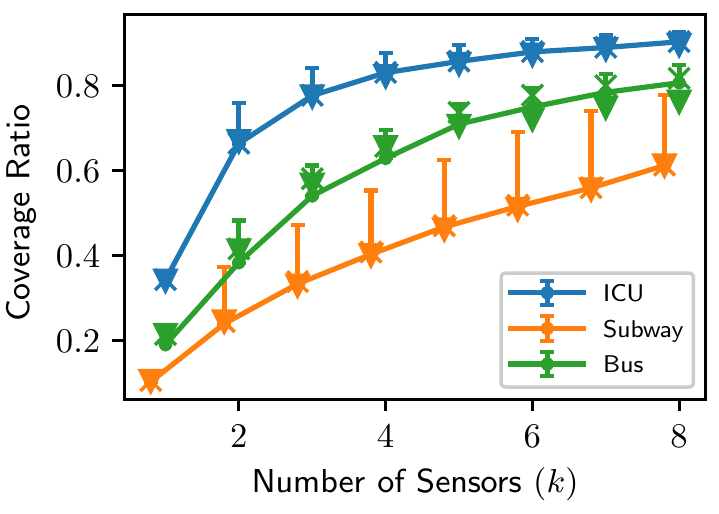}
    \includegraphics[width=.48\columnwidth, height=1in]{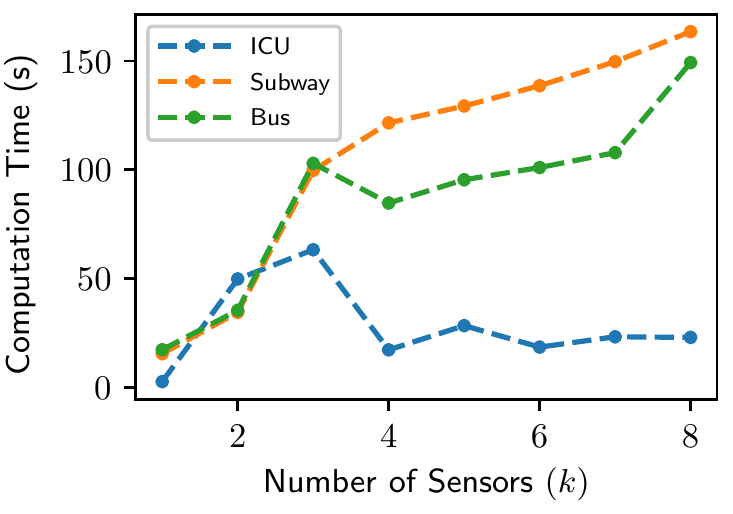}
    \caption{[left] Coverage ratio (lines) for Problem~\ref{p:3} returned by multiple methods. [right] Computation time used by running a coarse ILP plus local improvements.}
    \label{fig:computation-time-3}
\end{figure}


Lastly, we provide some additional visualization to help further demonstrate the structure of the problems. Fig.~\ref{fig:icu-comp} shows that Problems~\ref{p:2} and~\ref{p:3} induce different optimal distribution of sensors. Generally, Problems~\ref{p:2} tends to cause the sensors to be evenly spaced out. On the other hand, Problem~\ref{p:3} tends to balance between spacing out sensors and provide good cumulative coverage, which may require sensors to aggregate, which can be observed in Fig.~\ref{fig:bus}.

\begin{figure}[!ht]
\vspace{1mm}
    \centering
    \includegraphics[width = 0.35\columnwidth]{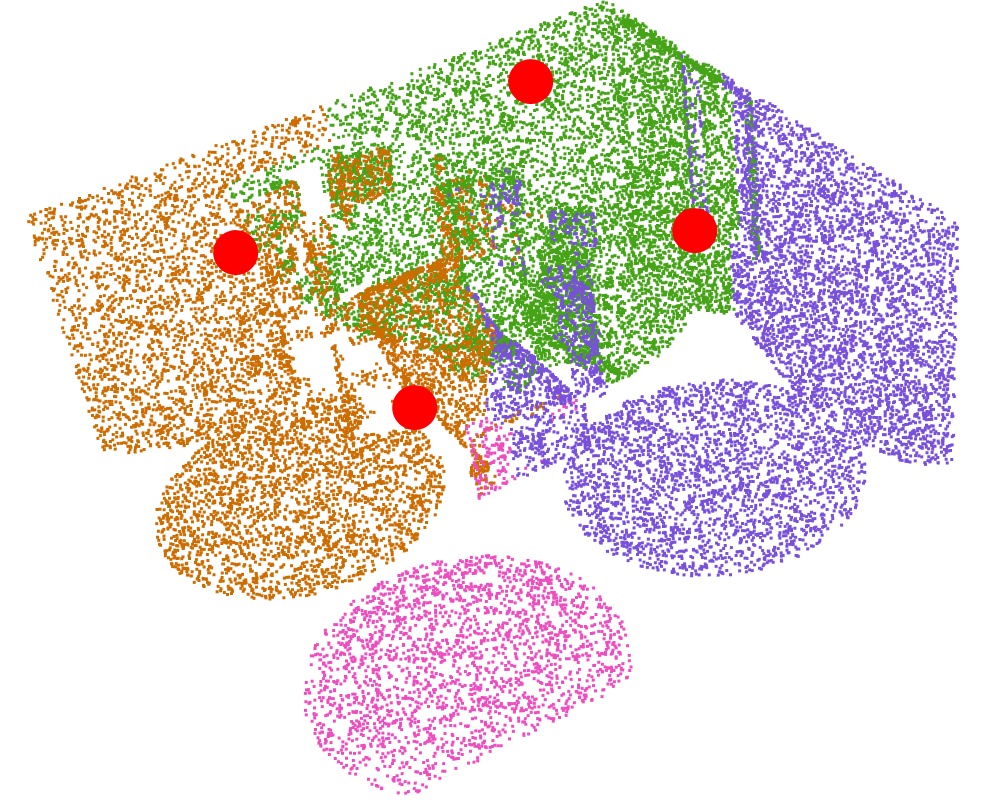}\hspace{3mm}
    \includegraphics[width = 0.35\columnwidth]{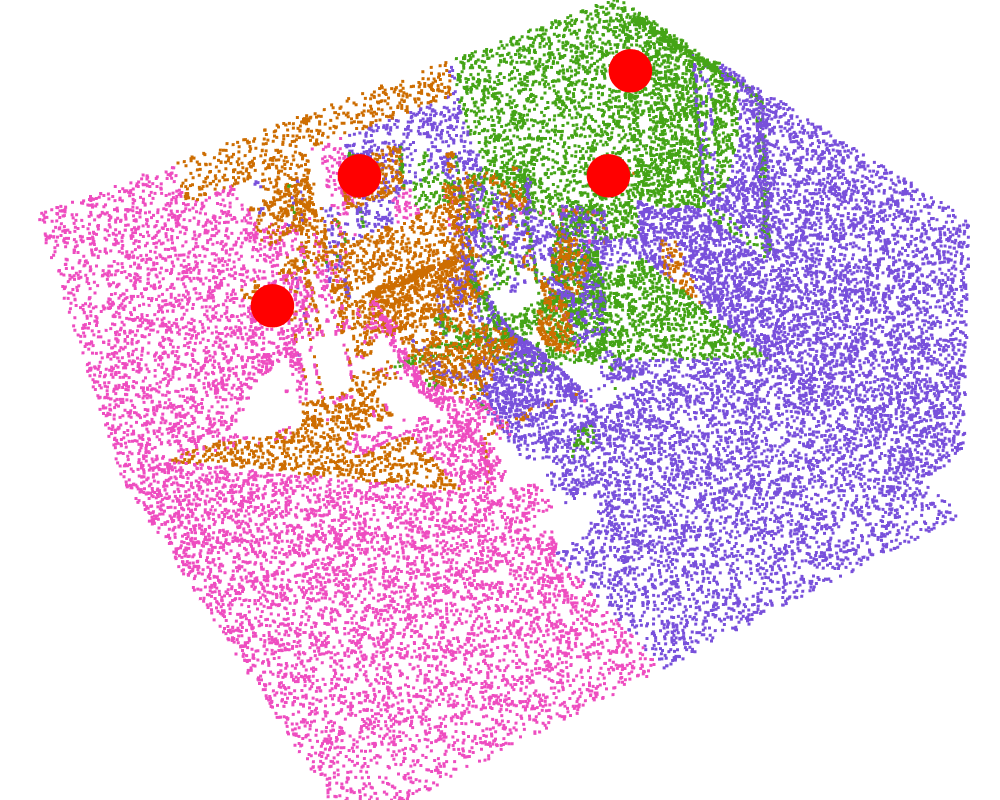}
\vspace{1mm}
    \caption{$4$ sensor ICU result for Problems~\ref{p:2} (left) and~\ref{p:3}(right).}
    \label{fig:icu-comp}
\end{figure}

\begin{figure}[!ht]
\vspace{-1mm}
    \centering
    \includegraphics[width = .25\columnwidth]{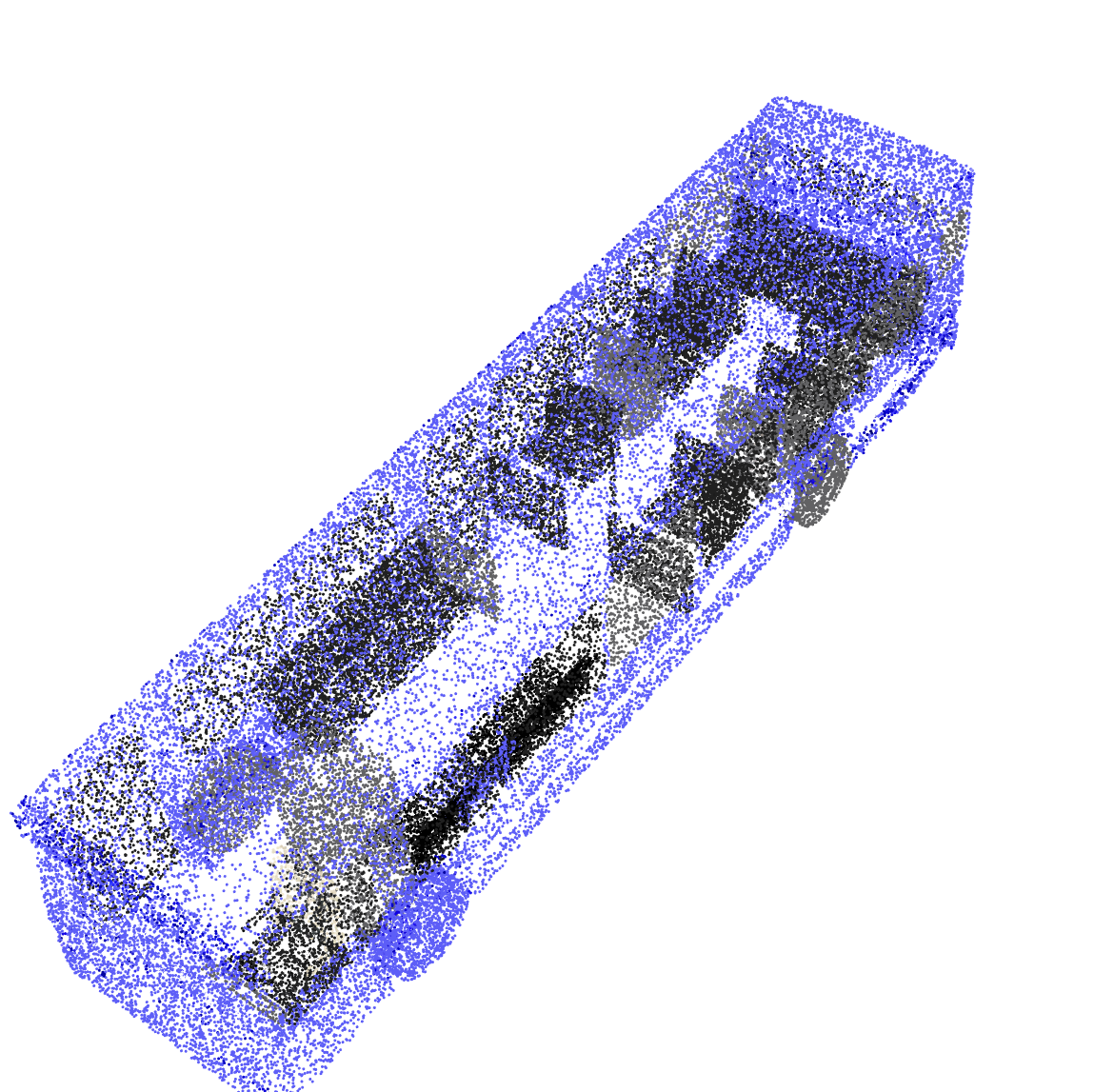}
    \includegraphics[width = .23\columnwidth]{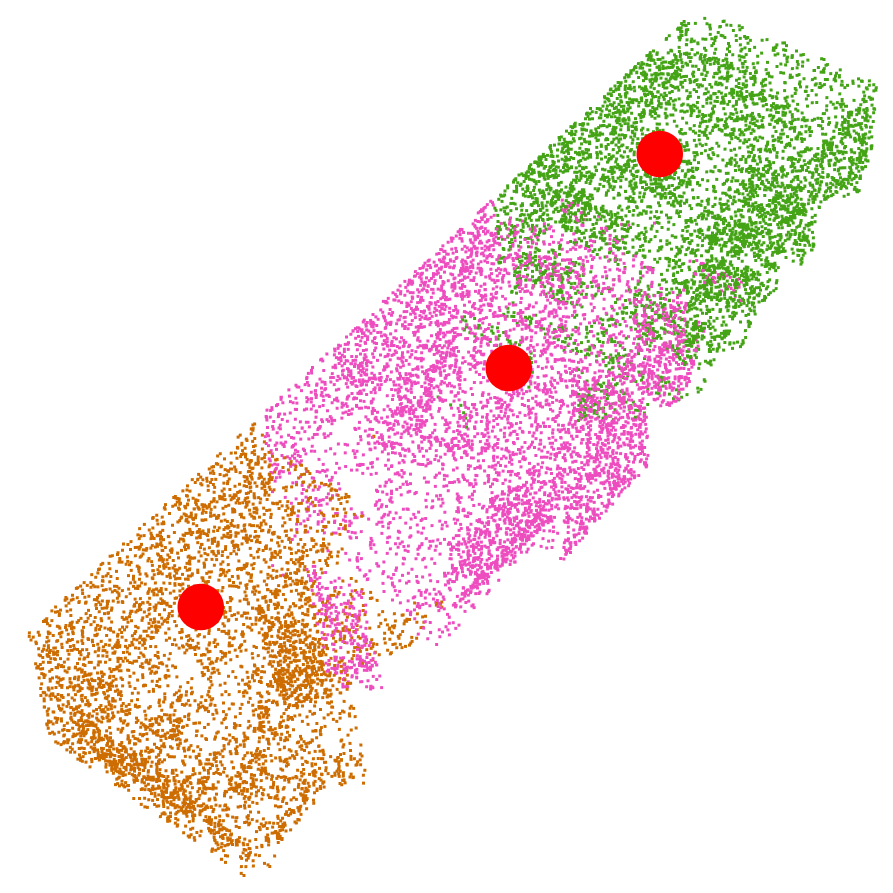}
    \includegraphics[width = .23\columnwidth]{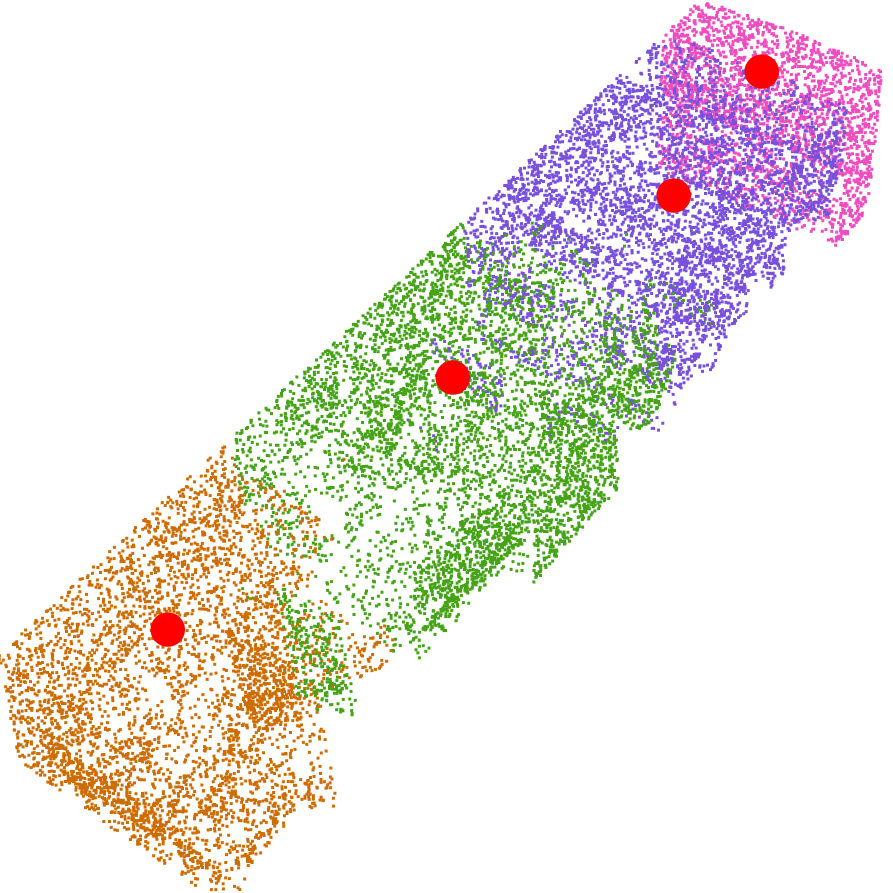}
    \includegraphics[width = .23\columnwidth]{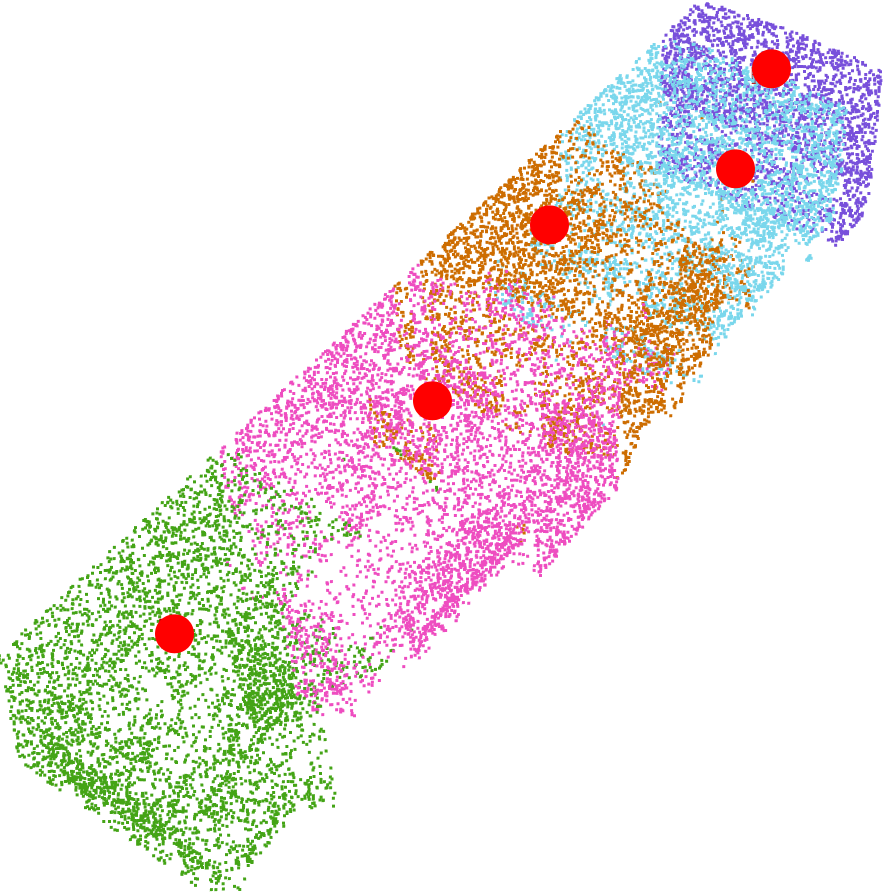}
    \caption{The coverage of bus (see through model on the left) using $3$ to $5$ sensors under Problem~\ref{p:3}. Aggregation of sensors at the front of the bus, which is structurally more complex, can be observed. }
    \label{fig:bus}
\end{figure}

\vspace{-1mm}
\section{Conclusion}\label{sec:conclusion}
\vspace{-1mm}
We have formulated a general Sensor Placement for Optimal Coverage (\spoc) problem with three concrete instantiations, each with distinct practical applications. We provide near-optimal methods for solving these challenging optimization problems and demonstrated their effectiveness with extensive evaluations. We are currently exploring real-world applications of our methods.  




\begin{thebibliography}{10}
\providecommand{\url}[1]{#1}
\csname url@samestyle\endcsname
\providecommand{\newblock}{\relax}
\providecommand{\bibinfo}[2]{#2}
\providecommand{\BIBentrySTDinterwordspacing}{\spaceskip=0pt\relax}
\providecommand{\BIBentryALTinterwordstretchfactor}{4}
\providecommand{\BIBentryALTinterwordspacing}{\spaceskip=\fontdimen2\font plus
\BIBentryALTinterwordstretchfactor\fontdimen3\font minus
  \fontdimen4\font\relax}
\providecommand{\BIBforeignlanguage}[2]{{%
\expandafter\ifx\csname l@#1\endcsname\relax
\typeout{** WARNING: IEEEtran.bst: No hyphenation pattern has been}%
\typeout{** loaded for the language `#1'. Using the pattern for}%
\typeout{** the default language instead.}%
\else
\language=\csname l@#1\endcsname
\fi
#2}}
\providecommand{\BIBdecl}{\relax}
\BIBdecl

\bibitem{o1987art}
J.~O'Rourke, \emph{Art Gallery Theorems and Algorithms}.\hskip 1em plus 0.5em
  minus 0.4em\relax Oxford University Press, 1987.

\bibitem{shermer1992recent}
T.~C. Shermer, ``Recent results in art galleries (geometry),''
  \emph{Proceedings of the IEEE}, vol.~80, no.~9, pp. 1384--1399, 1992.

\bibitem{o2017visibility}
J.~O'Rourke, ``Visibility,'' in \emph{Handbook of discrete and computational
  geometry, 3rd Ed.}, 2017, pp. 875--896.

\bibitem{howard2002mobile}
A.~Howard, M.~J. Matari{\'c}, and G.~S. Sukhatme, ``Mobile sensor network
  deployment using potential fields: A distributed, scalable solution to the
  area coverage problem,'' in \emph{Distributed Autonomous Robotic Systems
  5}.\hskip 1em plus 0.5em minus 0.4em\relax Springer, 2002, pp. 299--308.

\bibitem{cortes2004coverage}
J.~Cort{\'e}s, S.~Mart{\'\i}nez, T.~Karatas, and F.~Bullo, ``Coverage control
  for mobile sensing networks,'' \emph{IEEE Transactions on Robotics \&
  Automation}, vol.~20, no.~2, pp. 243--255, 2004.

\bibitem{martinez2007motion}
S.~Mart{\'\i}nez, J.~Cort{\'e}s, and F.~Bullo, ``Motion coordination with
  distributed information,'' \emph{IEEE Control Systems Magazine}, vol.~27,
  no.~4, pp. 75--88, 2007.

\bibitem{krause2008near}
A.~Krause, A.~Singh, and C.~Guestrin, ``Near-optimal sensor placements in
  gaussian processes: Theory, efficient algorithms and empirical studies,''
  \emph{Journal of Machine Learning Research}, vol.~9, no. Feb, pp. 235--284,
  2008.

\bibitem{schwager2009decentralized}
M.~Schwager, D.~Rus, and J.-J. Slotine, ``Decentralized, adaptive coverage
  control for networked robots,'' \emph{The International Journal of Robotics
  Research}, vol.~28, no.~3, pp. 357--375, 2009.

\bibitem{hollinger2013sampling}
G.~A. Hollinger and G.~S. Sukhatme, ``Sampling-based motion planning for
  robotic information gathering.'' in \emph{Robotics: Science and Systems},
  vol.~3, no.~5.\hskip 1em plus 0.5em minus 0.4em\relax Citeseer, 2013.

\bibitem{lozano1979algorithm}
T.~Lozano-P{\'e}rez and M.~A. Wesley, ``An algorithm for planning
  collision-free paths among polyhedral obstacles,'' \emph{Communications of
  the ACM}, vol.~22, no.~10, pp. 560--570, 1979.

\bibitem{lee1986computational}
D.~Lee and A.~Lin, ``Computational complexity of art gallery problems,''
  \emph{IEEE Transactions on Information Theory}, vol.~32, no.~2, pp. 276--282,
  1986.

\bibitem{thue1910dichteste}
A.~Thue, \emph{{\"U}ber die dichteste Zusammenstellung von kongruenten Kreisen
  in einer Ebene.}\hskip 1em plus 0.5em minus 0.4em\relax Christiania : Dybwad
  in Komm., 1910.

\bibitem{hales2005proof}
T.~C. Hales, ``A proof of the kepler conjecture,'' \emph{Annals of
  Mathematics}, vol. 162, no.~3, pp. 1065--1185, 2005.

\bibitem{drezner1995facility}
Z.~Drezner, \emph{Facility Location: A Survey of Applications and
  Methods}.\hskip 1em plus 0.5em minus 0.4em\relax Springer Verlag, 1995.

\bibitem{pavone2009equitable}
M.~Pavone, A.~Arsie, E.~Frazzoli, and F.~Bullo, ``Equitable partitioning
  policies for robotic networks,'' in \emph{2009 IEEE International Conference
  on Robotics and Automation}.\hskip 1em plus 0.5em minus 0.4em\relax IEEE,
  2009, pp. 2356--2361.

\bibitem{pierson2017adapting}
A.~Pierson, L.~C. Figueiredo, L.~C. Pimenta, and M.~Schwager, ``Adapting to
  sensing and actuation variations in multi-robot coverage,'' \emph{The
  International Journal of Robotics Research}, vol.~36, no.~3, pp. 337--354,
  2017.

\bibitem{schwager2009optimal}
M.~Schwager, B.~J. Julian, and D.~Rus, ``Optimal coverage for multiple hovering
  robots with downward facing cameras,'' in \emph{Proceedings IEEE
  International Conference on Robotics and Automation}.\hskip 1em plus 0.5em
  minus 0.4em\relax IEEE, 2009, pp. 3515--3522.

\bibitem{weber1929theory}
A.~Weber, \emph{Theory of the Location of Industries}.\hskip 1em plus 0.5em
  minus 0.4em\relax University of Chicago Press, 1929.

\bibitem{har2011geometric}
S.~Har-Peled, \emph{Geometric Approximation Algorithms}.\hskip 1em plus 0.5em
  minus 0.4em\relax American Mathematical Soc., 2011, no. 173.

\bibitem{feder1988optimal}
T.~Feder and D.~Greene, ``Optimal algorithms for approximate clustering,'' in
  \emph{Proceedings ACM Symposium on Theory of Computing}.\hskip 1em plus 0.5em
  minus 0.4em\relax ACM, 1988, pp. 434--444.

\bibitem{hochbaum1985best}
D.~S. Hochbaum and D.~B. Shmoys, ``A best possible heuristic for the k-center
  problem,'' \emph{Mathematics of Operations Research}, vol.~10, no.~2, pp.
  180--184, 1985.

\bibitem{gonzalez1985clustering}
T.~F. Gonzalez, ``Clustering to minimize the maximum intercluster distance,''
  \emph{Theoretical Computer Science}, vol.~38, pp. 293--306, 1985.

\bibitem{daskin2000new}
M.~S. Daskin, ``A new approach to solving the vertex p-center problem to
  optimality: Algorithm and computational results,'' \emph{Communications of
  the Operations Research Society of Japan}, vol.~45, no.~9, pp. 428--436,
  2000.

\bibitem{shamos1975closest}
M.~I. Shamos and D.~Hoey, ``Closest-point problems,'' in \emph{16th Annual
  Symposium on Foundations of Computer Science (FOCS 1975)}.\hskip 1em plus
  0.5em minus 0.4em\relax IEEE, 1975, pp. 151--162.

\bibitem{FenHanGaoYuRSS19}
S.~W. Feng, S.~D. Han, K.~Gao, and J.~Yu, ``Efficient algorithms for optimal
  perimeter guarding,'' in \emph{Robotics: Sciences and Systems}, 2019.

\bibitem{FenYu2020RAL}
S.~W. Feng and J.~Yu, ``Optimal perimeter guarding with heterogeneous robot
  teams: complexity analysis and effective algorithms,'' \emph{IEEE Robotics
  and Automation Letters}, 2020.

\bibitem{FenYuRSS20}
------, ``Optimally guarding perimeters and regions with mobile range
  sensors,'' in \emph{Robotics: Sciences and Systems}, 2020.

\bibitem{canny1987new}
J.~Canny and J.~Reif, ``New lower bound techniques for robot motion planning
  problems,'' in \emph{28th Annual Symposium on Foundations of Computer Science
  (FOCS 1987)}.\hskip 1em plus 0.5em minus 0.4em\relax IEEE, 1987, pp. 49--60.

\bibitem{vazirani2013approximation}
V.~V. Vazirani, \emph{Approximation Algorithms}.\hskip 1em plus 0.5em minus
  0.4em\relax Springer, 2013, ch.~5.

\bibitem{cgal:aabb-20b}
\BIBentryALTinterwordspacing
P.~Alliez, S.~Tayeb, and C.~Wormser, ``3d fast intersection and distance
  computation (aabb tree),'' in \emph{{CGAL} User and Reference Manual},
  {5.1}~ed.\hskip 1em plus 0.5em minus 0.4em\relax {CGAL Editorial Board},
  2020. [Online]. Available:
  \url{\url{https://doc.cgal.org/5.1/Manual/packages.html\#PkgAABBTree}}
\BIBentrySTDinterwordspacing

\bibitem{welzl1991smallest}
E.~Welzl, ``Smallest enclosing disks (balls and ellipsoids),'' in \emph{New
  Results and New Trends in Computer Science}.\hskip 1em plus 0.5em minus
  0.4em\relax Springer, 1991, pp. 359--370.

\bibitem{Mark1997computation}
M.~Berg, de, O.~Cheong, M.~Kreveld, van, and M.~Overmars,
  \emph{\BIBforeignlanguage{eng}{Computational Geometry: Algorithms and
  Applications}}.\hskip 1em plus 0.5em minus 0.4em\relax Springer, 2008, ch.
  4.7.

\bibitem{gurobi}
\BIBentryALTinterwordspacing
L.~Gurobi~Optimization, ``Gurobi optimizer reference manual,'' 2020. [Online].
  Available: \url{http://www.gurobi.com}
\BIBentrySTDinterwordspacing

\end{thebibliography}

\end{document}